\documentclass[letterpaper, 10 pt, conference]{ieeeconf}  

\IEEEoverridecommandlockouts                              

\usepackage{hyperref}
\hypersetup{
    colorlinks=true,
    linkcolor=black,    
    urlcolor=blue,
    }
\usepackage{algorithm}
\usepackage{algpseudocode}
\usepackage{graphicx}
\usepackage{subcaption}
\usepackage{booktabs}
\usepackage{multirow}
\usepackage{epsfig} 
\usepackage{amsmath} 
\usepackage{amssymb}  
\usepackage{mathrsfs}
\usepackage{cleveref}
\usepackage{color}
\usepackage{calc}
\usepackage{dsfont}
\usepackage{bm}
\usepackage{float}
\usepackage{url}
\newtheorem{proposition}{Proposition}

\usepackage{xspace}
\usepackage[dvipsnames]{xcolor}
\usepackage[font=small]{caption}

\def \FigureAbbreviaition {Fig.}
\newcommand{\figref}[1]{\FigureAbbreviaition~\ref{#1}}

\usepackage{amsfonts} 
\title{\LARGE \bf
Safety-Critical Control for Smoothed Implicit Contact Dynamics
}
\author{Haegu Lee, Yitaek Kim and Christoffer Sloth
\thanks{All authors are with The Maersk Mc-Kinney Moller Institute, University of Southern Denmark, Denmark {\tt\small \{haeg, yik, chsl\}@mmmi.sdu.dk}}}
\begin{document}
\maketitle
\thispagestyle{empty}
\pagestyle{empty}
\begin{abstract}
Smoothed implicit contact dynamics enables gradient-based planning and control for contact-rich tasks without predefined mode sequences.
However, safety-critical control remains challenging because implicit contact dynamics makes safety-filter design nontrivial.
The smoothing parameter $\kappa$ relaxes contact complementarity constraints, which makes the dynamics smooth but affects the contact force.
This paper provides a safety-filtering framework for smoothed implicit contact dynamics. We first derive a discrete-time control barrier function (CBF) constraint using a first-order Taylor approximation of the implicitly defined contact force.
We show that, although reducing $\kappa$ can improve local force-approximation accuracy, the resulting closed-loop force-constraint violations can vary non-monotonically with $\kappa$. Motivated by this observation, we introduce boundary-focused rollouts that screen candidate $\kappa$ values by comparing the predicted safety margin with the observed one-step under-prediction. We then robustly tighten the predicted CBF constraint with a fixed margin to account for residual force under-prediction. Simulations on four contact-rich systems show that the proposed method eliminates force violations observed under a standard CBF. Project website : \url{https://contact-cbf.github.io/}
\end{abstract}

\section{Introduction}\label{sec:introduction}

Robotic systems that undergo intermittent contact interaction with the environment are difficult to plan because their contact modes change over time \cite{patel2019contact, hogan2020feedback}. Hybrid contact controllers may explicitly resolve contact modes online \cite{hogan2020reactive}, which can be nontrivial for contact-rich tasks. However, implicit contact dynamics provides an attractive alternative by representing contact within a unified dynamics framework without requiring predefined mode sequences. Implicit contact dynamics has facilitated trajectory planning through contact in several recent works \cite{posa2014direct, kim2025contact, shirai2024robust, lee2025trajectory}.

Contact dynamics are inherently nonsmooth and non-differentiable, which makes them difficult to handle with gradient-based methods. A common approach is to smooth the contact dynamics to make it suitable for gradient-based optimization. Existing smoothing methods aim to make the dynamics differentiable \cite{suh2022bundled, howell2022dojo}. In particular, smoothed contact dynamics have been successfully applied to gradient-based planning through contact \cite{pang2023global, suh2025dexterous, kurtz2026inverse}. However, these advances have primarily focused on differentiable simulation and planning, rather than safety-critical control.

 \begin{figure}[t]
    \centering
    \includegraphics[width=1\linewidth]{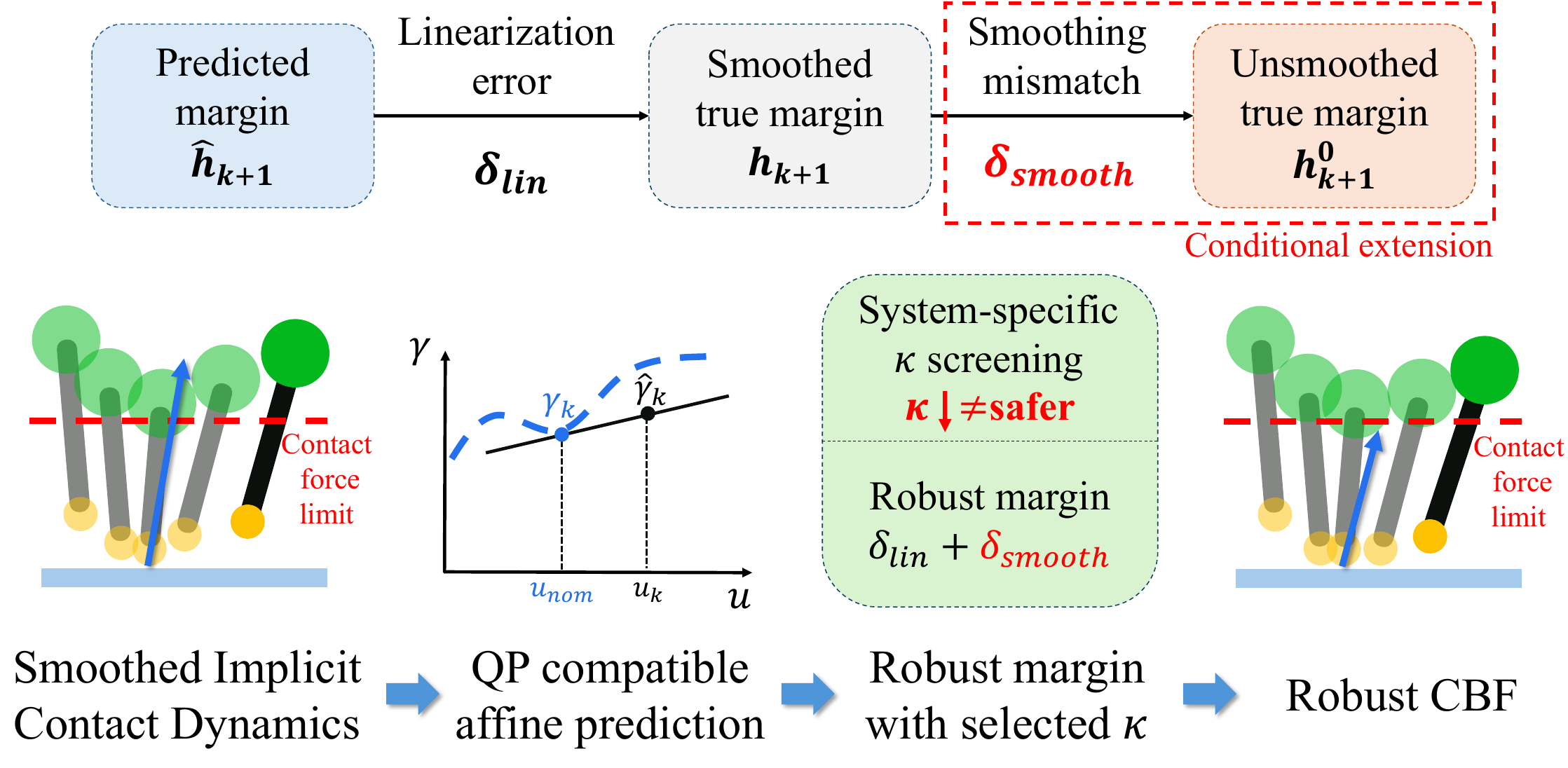}
    \caption{Proposed robust CBF framework for smoothed implicit contact dynamics. A local affine prediction enables a CBF-QP, while $\kappa$ screening and robust tightening address linearization error, with a conditional extension to the unsmoothed model.}
    \label{fig:hero_figure}
\end{figure}

Safety guarantees are often required in contact-rich control, particularly when tasks involve fragile objects, uncertain contact conditions, and human-robot interaction. Recent studies have incorporated force-related safety using explicit contact models through interaction-force constraints and CBF-based safety filtering \cite{wang2025guarding, vinter2024safe, kim2025robust}. Robust CBF formulations have also been developed
to account for disturbances and uncertainty in the system dynamics \cite{xu2015robustness, jankovic2018robust}. However, these approaches rely on explicit models whose dynamics are often control-affine and therefore readily compatible with standard CBF formulations. 

Recent advances in differentiable physics engines have enabled efficient gradient-based simulation and optimization for contact-rich robotic systems \cite{hu2019difftaichi, geilinger2020add, lee2023differentiable}. Among these approaches, formulations based on Nonlinear Complementarity Problems (NCPs) with second-order cone constraints provide accurate and differentiable simulation of implicit contact dynamics \cite{howell2022dojo}. Although these models have shown strong performance in trajectory optimization and planning through contact \cite{le2024fast}, formal safety guarantees for the strict enforcement of contact-force constraints remain limited. 
This is because implicit contact models are generally non-affine in the control input. As a result, standard CBF formulations for control-affine systems cannot be applied directly \cite{ames2019control, agrawal2017discrete, zeng2021safety}. Although several studies have considered safety for non-affine systems \cite{son2019safety, seo2022safety, xiao2023safe}, they do not directly address implicit contact dynamics governed by complementarity constraints. Therefore, ensuring contact-force safety under implicit contact dynamics remains an open challenge.

In this paper, we make three main contributions:
\begin{enumerate}
    \renewcommand{\labelenumi}{\arabic{enumi})}
    \item We formulate a discrete-time CBF framework for contact-force safety in smoothed implicit contact dynamics, with the safety set defined by contact-force constraints.

    \item We characterize how the smoothing parameter $\kappa$ affects the CBF safety condition by relating the predicted safety margin to the one-step local force-approximation error, showing that a smaller $\kappa$ can improve local accuracy while degrading safety-filter reliability.

    \item We propose a robust safety filter that tightens the predicted CBF condition with a conservative approximation-error margin to bound the smoothed contact force, with a conditional extension to the unsmoothed force given a valid smoothing-mismatch bound.
\end{enumerate}
Simulations on multiple contact-rich systems show that the proposed method reduces contact-force constraint violations relative to a standard CBF.
\section{Preliminaries and Problem Formulation}\label{sec:preliminaries}
In this section, we briefly review the implicit contact dynamics model and discrete-time control barrier functions used in the proposed safety filter.
\subsection{Contact Dynamics Model} 
Accurate modeling of contact forces is important when robots interact with environments. This paper adopts the complementarity-based contact dynamics model of \cite{howell2022dojo}, in which contact and friction are modeled using a Nonlinear Complementarity Problem (NCP).

Let \(\phi : \mathcal{Z} \to \mathbb{R}\) be the signed-distance function, where \(\mathcal{Z}\) denotes the configuration space. The normal contact force at time step $k$ is denoted by $\gamma_{k}$. In the discrete-time setting, a central-path parameter $\kappa > 0$ is introduced to relax the hard complementarity condition. The resulting unilateral contact conditions at configuration $z\in\mathcal{Z}$ are
\begin{subequations}\label{eq:penetration_constraints}
\begin{align}
    \phi(z) &\ge 0, \label{eq:relaxed_contact_gap}\\
    \gamma &\ge 0, \label{eq:relaxed_contact_force}\\
    \gamma \phi(z) &= \kappa. \label{eq:relaxed_contact_central_path}
\end{align}
\end{subequations}
These conditions enforce non-penetration and unilateral contact while replacing the hard complementarity relation ($\kappa=0$) with a smooth central-path condition ($\kappa > 0$). Fig.~2 illustrates the resulting effect on local contact sensitivity in a one-dimensional box contact example. In the nonsmooth model, the contact-force sensitivity with respect to the vertical input force changes abruptly near lift-off, whereas the smoothed model produces a differentiable transition whose sensitivity depends on $\kappa$. 

Frictional contact is modeled using the relative tangential velocity $v_{k+1}$ between the objects in contact and tangential contact impulse $\mathbf{\beta}_{k+1}$ at the contact point. In the discrete-time setting, Coulomb friction is derived from the maximum dissipation principle, which yields
\begin{equation}
\begin{aligned}
\min_{\mathbf{\beta}_{k+1}} \quad & v_{k+1}^\top \mathbf{\beta}_{k+1} \\
\text{s.t.} \quad & \|\mathbf{\beta}_{k+1}\|_2 \le \mu \gamma_{k+1},
\end{aligned}
\label{eq:friction_constraint}
\end{equation}
where $\mu$ is the coefficient of friction. Here, $v_{k+1}=v(z_k,z_{k+1})$ is determined by the current and next configurations and is not an independent variable.

Together with the unilateral contact condition \eqref{eq:penetration_constraints} and the friction-cone constraint \eqref{eq:friction_constraint}, this defines the contact update at time step $k$ as a nonlinear complementarity problem. 
Let
\begin{equation}
\theta_k := (z_{k-1}, z_k, u_k), \,\,\,
w_k := (z_k,z_{k+1}, \gamma_{k+1}, \mathbf{\beta}_{k+1}),\label{eqn:theta+w}
\end{equation}
where $\theta_k$ denotes the problem data at time step $k$, and $w_k$ denotes the corresponding solution.

The relaxed implicit contact dynamics can be written compactly as
\begin{equation}
r_{\kappa}(w_k;\theta_k)=0.
\label{eq:relaxed_ncp}
\end{equation}
Let $w_k^\star$ denote a solution to \eqref{eq:relaxed_ncp}. If $r$ is continuously differentiable and $\partial r/\partial w_k$ is nonsingular at $w_k^\star$, the implicit function theorem \cite{dini1907lezioni} yields a locally unique differentiable solution map $w_k^\star=\psi_\kappa(\theta_k)$, with 
\begin{equation}
\frac{\partial w_k^\star}{\partial \xi}
=
-
\left(
\frac{\partial r}{\partial w_k}
\right)^{-1}
\frac{\partial r}{\partial \xi},
\qquad
\xi \in \{\theta_k,\kappa\}. \label{eq:implicit_diff_eq}
\end{equation}
The contact force component of this sensitivity later provides the local input derivative needed to formulate QP-compatible CBF constraints. 
 \begin{figure}[t]
    \centering
    \includegraphics[width=0.85\linewidth]{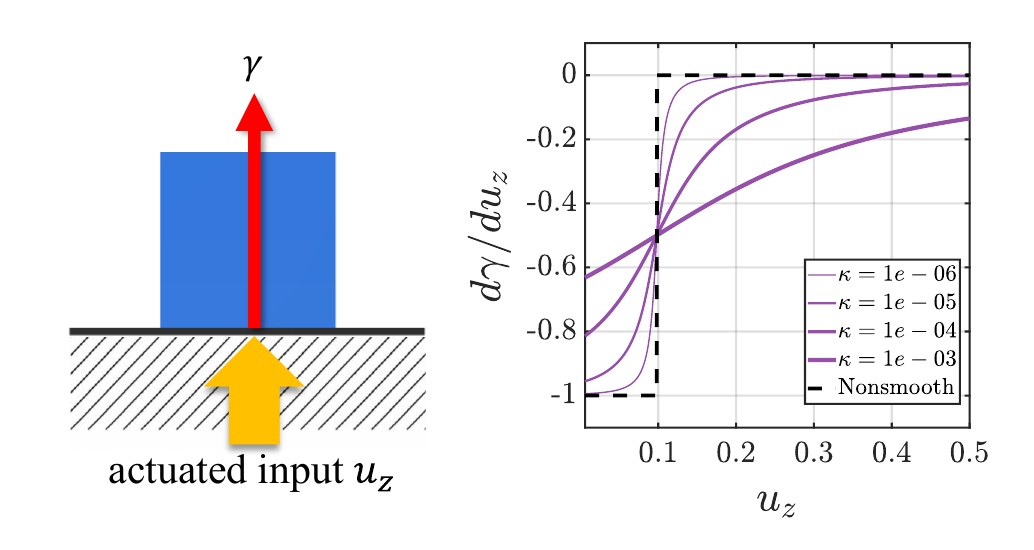}
    \caption{Gradients are computed for different values of the central-path parameter $\kappa$. The dynamics involve a box of mass 0.01 kg resting on a flat surface in the XZ plane. The applied force was gradually increased from 0~N to 0.5~N.}
    \label{fig:kappa_smoothing}
\end{figure}
 
\subsection{Control Barrier Functions}
Consider the following control system with explicit dynamics
\begin{equation}
x_{k+1} = f(x_k,u_k) \label{system}
\end{equation}
where $x_k \in \mathcal{X} \subset \mathbb{R}^n$ consists of $z_k$ and $v_k$, and \(f : \mathcal{X}\times \mathcal{U} \rightarrow \mathbb{R}^n\) is locally Lipschitz. The safe set is defined as the \(0\)-superlevel set of a continuously differentiable function \(h: \mathbb{R}^n \rightarrow \mathbb{R}\):
\begin{equation}
\mathcal{S} = \{ x \in \mathcal{X} \mid h(x) \ge 0\}.\label{eqn:safe_set_def}
\end{equation}
For system \eqref{system}, the continuously differentiable function \(h\) is a discrete-time control barrier function \cite{zeng2021safety} if \(\frac{\partial h}{\partial x} \neq 0\) for all \(x_k \in \partial \mathcal{S}\), and there exists an extended class $\mathcal{K}$ function $\alpha$ such that, for all $x_k$, there exists $u_k\in\mathcal{U}$ satisfying
\begin{equation}
h_{k+1} - h_k \ge -\alpha(h_k),
\label{discrete-CBF}
\end{equation}
where $h_k := h(x_k)$ and $h_{k+1} := h(x_{k+1})$. This implies that the system is safe with respect to $\mathcal{S}$ for any $x_0\in \mathcal{S}$. Following \cite{zeng2021safety}, we define $\alpha(h_k) := \alpha h_k$ where $\alpha$ is chosen as a constant such that $0<\alpha \leq 1.$ 
\subsection{Problem Formulation}
This paper considers contact-force safety for smoothed implicit contact dynamics. At each time step $k$, given the current state and a nominal input $u_k^{\mathrm{nom}}$, let $\gamma_{k+1}^{\mathrm{nom}} := \gamma_{k+1}(u_k^{\mathrm{nom}};\kappa)$ denote the corresponding next-step contact force. The goal is to compute a safety-filtered input $u_k \in \mathcal{U}$ that minimally deviates from $u_k^{\mathrm{nom}}$ while satisfying $\gamma_{k+1}(u_k;\kappa) \leq \gamma_{\max}$.

To support this objective, we analyze how the central-path parameter $\kappa$ affects the relaxed contact response and safety margin, and calibrate a robust margin over a task-specific operating set $\mathcal{D}$. As summarized in Fig.~\ref{fig:workflow}, a fixed rollout set $\mathcal{P}$ is generated offline from task-specific perturbations and used for $\kappa$ selection and margin calibration. The nominal signed-distance profile $\phi_k^{\mathrm{nom}}:=\phi(z_k^{\mathrm{nom}})$ illustrates the contact evolution represented by these rollouts. The selected parameters are fixed during online safety filtering, which is independent of how the nominal input is generated. We also consider a conditional extension to the unsmoothed dynamics given a valid smoothing-mismatch bound.

\begin{figure}[t]
    \centering
    \includegraphics[width=0.88\linewidth]{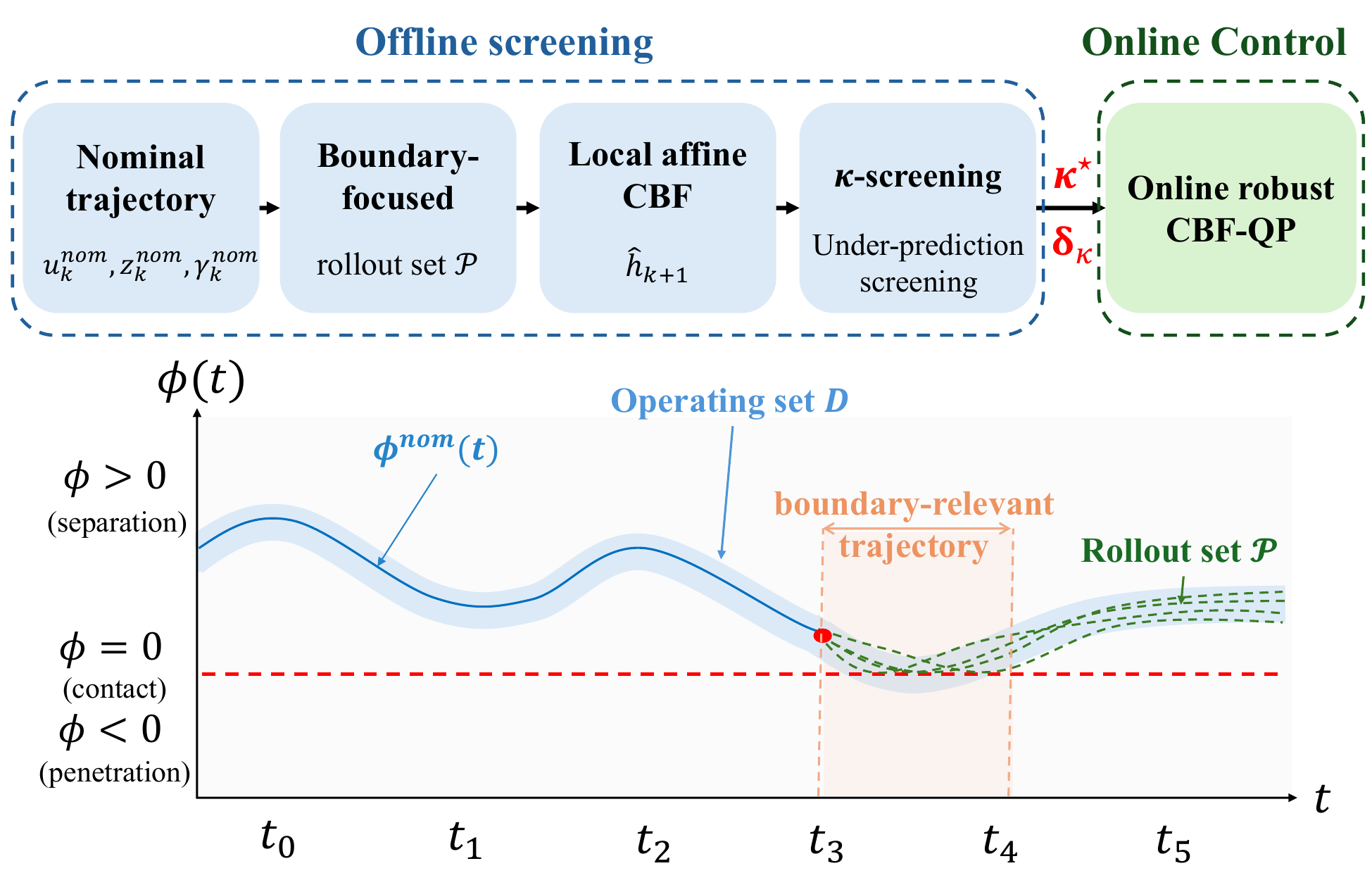}
    \caption{Offline construction of the boundary-focused rollout set $\mathcal{P}$ and online robust CBF-QP filtering.}
    \label{fig:workflow}
\end{figure}
\section{Safety-Critical Control with Implicit Contact Dynamics}\label{sec:safety}
CBF frameworks typically rely on explicit control-affine system dynamics, allowing safety constraints to be written as quadratic programs (QPs). In our setting, the contact force is defined implicitly through \eqref{eq:relaxed_ncp}, so the standard CBF-QP cannot be applied directly. We therefore use a local linear approximation of the solution to derive a QP-compatible constraint.

\subsection{Approximation of Implicit Contact Dynamics}
At this stage, two approaches are possible: 1) Combine the CBF-based safety filter directly with the implicit contact dynamics, which leads to a nonlinear optimization problem that preserves the implicit model structure; 2) Locally linearize the implicit dynamics around a nominal trajectory, which yields an affine approximation suitable for a standard CBF-QP. We adopt a local approximation of \eqref{eq:relaxed_ncp} as it is suited for real-time implementation. 

Let $\theta_k^{\mathrm{nom}} =(z_{k-1}^{\mathrm{nom}},z_k^{\mathrm{nom}},u_k^{\mathrm{nom}})$ denote the problem data at the nominal input and $w_k^{\mathrm{nom}}$ be the solution such that $r_\kappa(\theta_k^{\mathrm{nom}},w_k^{\mathrm{nom}})=0$.
Then, the solution $w_k$ satisfying $r_\kappa(w_k;\theta_k)=0$  can be locally approximated using a first-order Taylor expansion  based on the implicit sensitivity in \eqref{eq:implicit_diff_eq} as
\begin{align}
    \hat{w}_k \approx w_k^{\mathrm{nom}}+\frac{\partial w}{\partial \theta}(\theta_k-\theta_k^{\mathrm{nom}}),
    \label{eqn:taylor_full}
\end{align}
where the sensitivity is evaluated at $(w_k^{\mathrm{nom}},\theta_k^{\mathrm{nom}})$.

Using the differentiable solution map $w_k=\psi_{\kappa}(\theta_k)$, we define the next-step safety margin as
\begin{align}
    h_{k+1} := (h_w\circ\psi_\kappa)(\theta_k) = h_w(w_k),
    \label{eq:implicit_barrier}
\end{align}
where $h_w$ is a continuously differentiable function \(h_w: \mathcal{W} \rightarrow \mathbb{R}\). Finally, an affine approximation of $h_{k+1}$ in \eqref{discrete-CBF} is defined as
\begin{equation}
\hat{h}_{k+1}\approx h_w(w_k^{\mathrm{nom}})+\frac{\partial h_w}{\partial w}\frac{\partial w}{\partial \theta}(\theta_k-\theta_k^{\mathrm{nom}}).
\label{eqn:h approx}
\end{equation}
For safety-critical control, only the control input $u_k$ changes in the problem data $\theta_k$ (see \eqref{eqn:theta+w}); thus, \eqref{eqn:h approx} reduces to
\begin{align}
    \hat{h}_{k+1}\approx h_{k+1}^{\mathrm{nom}}+J_k(u_k-u_k^{\mathrm{nom}}),
    \label{eqn:taylor_w}
\end{align}
where $h_{k+1}^{\mathrm{nom}}:=h_w(w_k^{\mathrm{nom}})$ and 
$J_k:=\left.\frac{\partial h_w}{\partial w}\frac{\partial w}{\partial u_k} \right|_{(w_k^{\mathrm{nom}},\theta_k^{\mathrm{nom}})}.$

The safety-critical input is computed by solving the following CBF-QP:
\begin{equation}
\begin{aligned}
u_k^\star={}&\underset{u_k \in \mathcal{U}}{\arg\min}\; \frac{1}{2}\left\|u_k-u_k^{\mathrm{nom}}\right\|^2 \\
\text{s.t.}\quad &
h^{\mathrm{nom}}_{k+1}+J_k\left(u_k-u_k^{\mathrm{nom}}\right)-h_k \ge-\alpha h_k.
\end{aligned}
\label{eq:cbf_lin2}
\end{equation}
We now specialize the general construction to contact-force safety by choosing \(h_w(w):=\gamma_{\max}-\gamma\). With this choice, the current and nominal next-step safety margins are
\begin{equation}
    h_k:=\gamma_{\max}-\gamma_k,\quad
    h_{k+1}^{\mathrm{nom}}
    := h_w(w_k^{\mathrm{nom}})
    =\gamma_{\max}-\gamma_{k+1}^{\mathrm{nom}}.
\end{equation}

The local force-approximation error is accounted for in Sec.~\ref{sec:robust_cbf} through robust tightening.

\begin{figure}[t]
    \centering
    \includegraphics[width=0.88\linewidth]{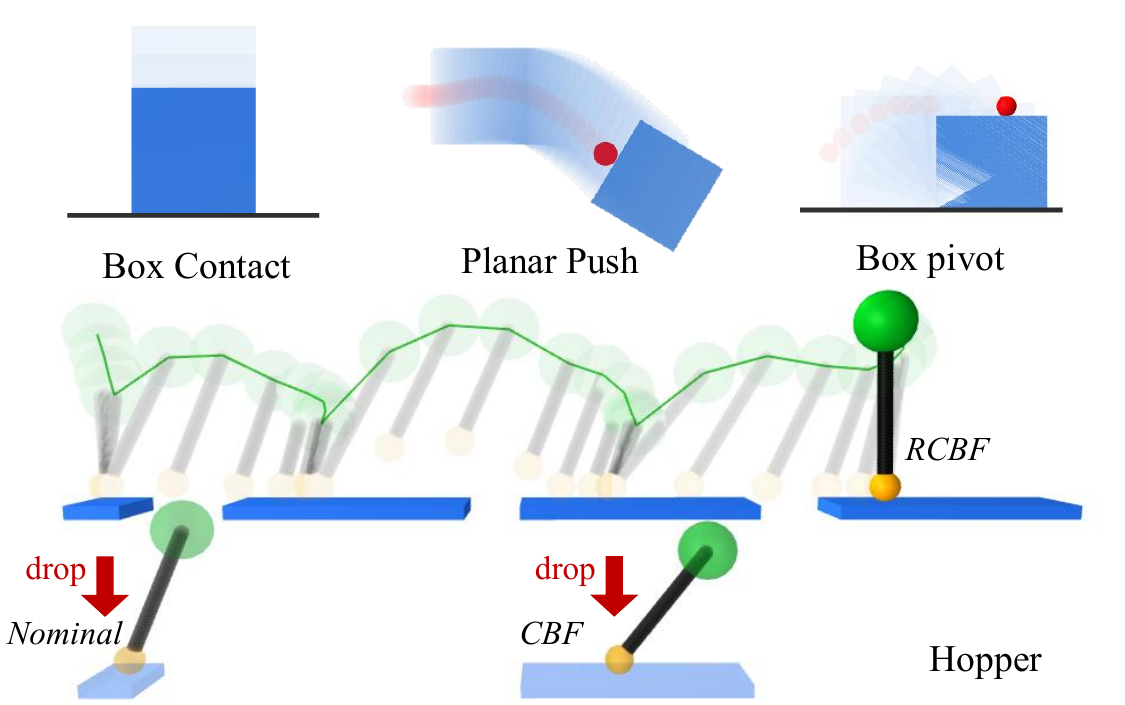}
    \caption{Contact-force safety examples under smoothed implicit contact dynamics. Red circles denote independently actuated pushers in frictional contact; the bottom row shows forward hopping over a possible terrain drop.}
    \label{fig:hero}
\end{figure}
\subsection{$\kappa$-Dependent Safety Margin Analysis}
The parameter $\kappa$ affects both the accuracy of the local contact-force approximation and the resulting contact response. Since the barrier function is defined in terms of the contact force, the safety margin also depends on $\kappa$ through $\gamma_{k+1}(u_k;\kappa)$. The corresponding safety-margin approximation error is
\begin{equation}
\Delta h_k := h_{k+1}-\hat h_{k+1} = -\epsilon_{\gamma,k}(u_k;\kappa).
\label{eq:approximation_error}
\end{equation}
We use the one-dimensional box-contact example in \figref{fig:hero} to illustrate how these two effects jointly influence safety. \figref{fig:lin_error} shows that the one-step approximation error decreases as $\kappa$ becomes smaller. However, as shown in Fig.~\ref{fig:kappa_cbf_test}, among the four tested values of $\kappa$ at $\gamma_{\max}=0.25\,\mathrm{N}$, only $\kappa = 10^{-4}$ remains within the force limit, whereas both smaller and larger values lead to force violations.

 \begin{figure}[t]
    \centering
    \includegraphics[width=0.8\linewidth]{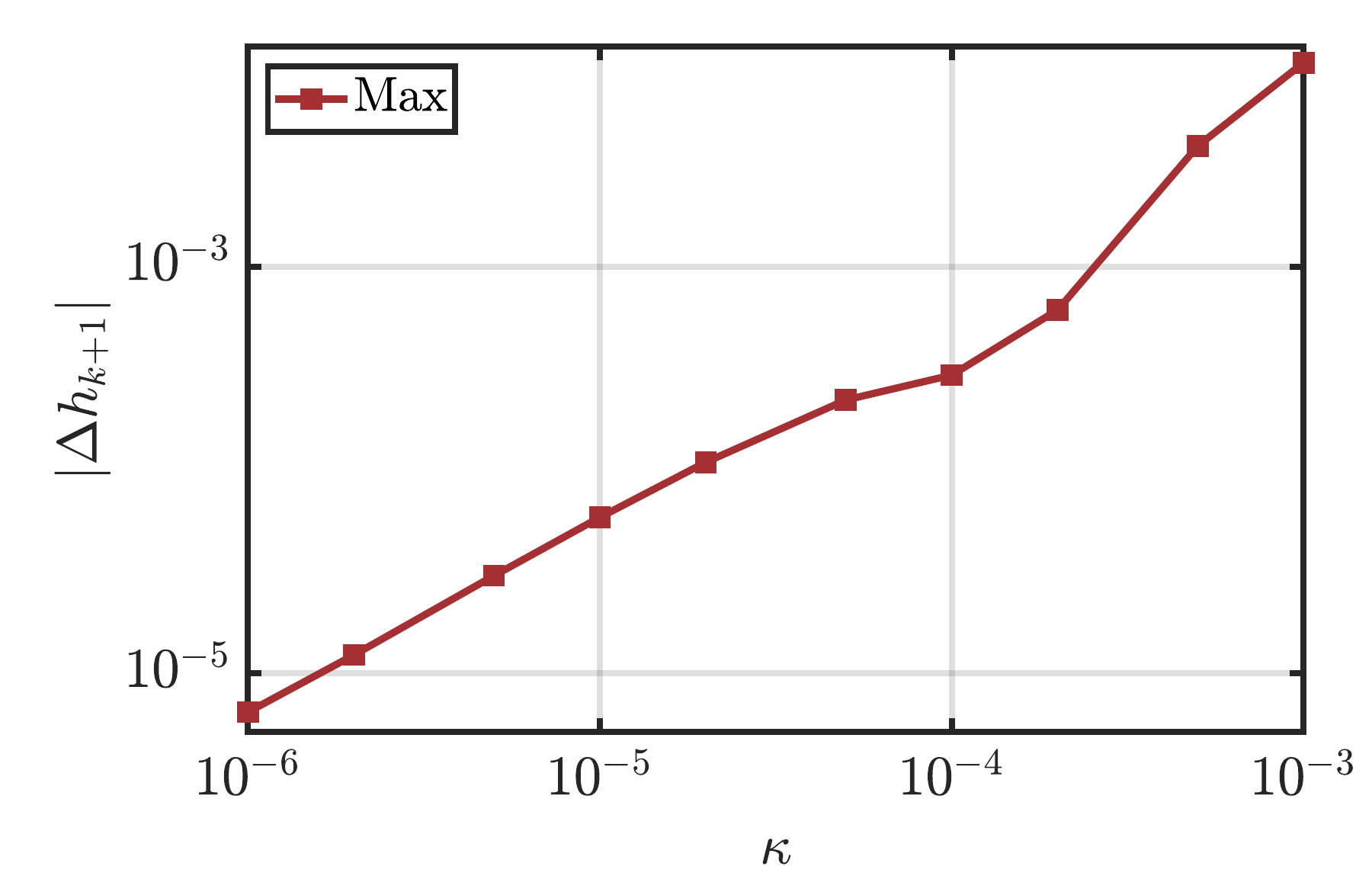}
    \caption{Maximum absolute safety-margin approximation error $|\Delta h_k|$ versus $\kappa$ for the 1D box-contact rollout, evaluated over steps with $\gamma_{k+1}>0$.}
    \label{fig:lin_error}
\end{figure}

The finer sweep in \figref{fig:kappa_sweep_test}(a) shows that the non-violation region appears over an interval of $\kappa$ near $10^{-4}$. This behavior can be understood from the interaction between the predicted safety margin and the approximation error. While the CBF-QP enforces the linearized constraint using $\hat{h}_{k+1}$, true safety depends on 
\begin{equation}
h_{k+1} = \hat{h}_{k+1} + \Delta h_k. \label{eq:safety_h}
\end{equation}
Consequently, a safety violation occurs whenever
\begin{equation}
\hat{h}_{k+1} + \Delta h_k < 0. 
\label{eq:safe_violation_condition}
\end{equation}
Although decreasing $\kappa$ reduces the approximation error, the empirical results show that safety does not necessarily improve monotonically as $\kappa$ decreases. To interpret this behavior, we examine the local
central-path relation along the $\kappa$-dependent closed-loop solution.
Consider a range of positive $\kappa$ values over which the predicted force constraint remains active. Assume that the resulting closed-loop quantities are differentiable with respect to $\kappa$. Using the signed-distance function in \eqref{eq:relaxed_contact_gap}, define the contact gap as
\begin{equation}
    \Phi_{k+1}(\kappa) := \phi\bigl(z_{k+1}^\star(\theta_k(\kappa),\kappa)\bigr),
\end{equation}
where $z_{k+1}^\star(\theta_k(\kappa),\kappa)$ is the next-step state component induced by the smoothed implicit contact dynamics. 
Suppose that the contact force and the closed-loop gap satisfy the local smoothed central path relation
\begin{align}
    \gamma_{k+1}(\kappa)>0, \Phi_{k+1}(\kappa)>0,\\
    \gamma_{k+1}(\kappa)\Phi_{k+1}(\kappa)=\kappa.
\end{align}
Define the true next-step safety margin as
\begin{equation}
    h_{k+1}(\kappa):=\gamma_{\max} - \gamma_{k+1}(\kappa),
\end{equation}
and define the logarithmic closed-loop gap sensitivity as
\begin{equation}
    A_\phi(\kappa)
    :=
    \frac{d\log \Phi_{k+1}(\kappa)}
    {d\log \kappa}
    =
    \frac{\kappa}{\Phi_{k+1}(\kappa)}
    \frac{d\Phi_{k+1}(\kappa)}{d\kappa}.
\end{equation}
Log-differentiating the central-path relation gives
\begin{equation}
    \frac{d h_{k+1}}{d\log \kappa}
    =
    \gamma_{k+1}(\kappa)
    \left(A_\phi(\kappa)-1\right),
\end{equation}

\begin{figure}[t]
    \centering
    \includegraphics[width=0.95\linewidth]{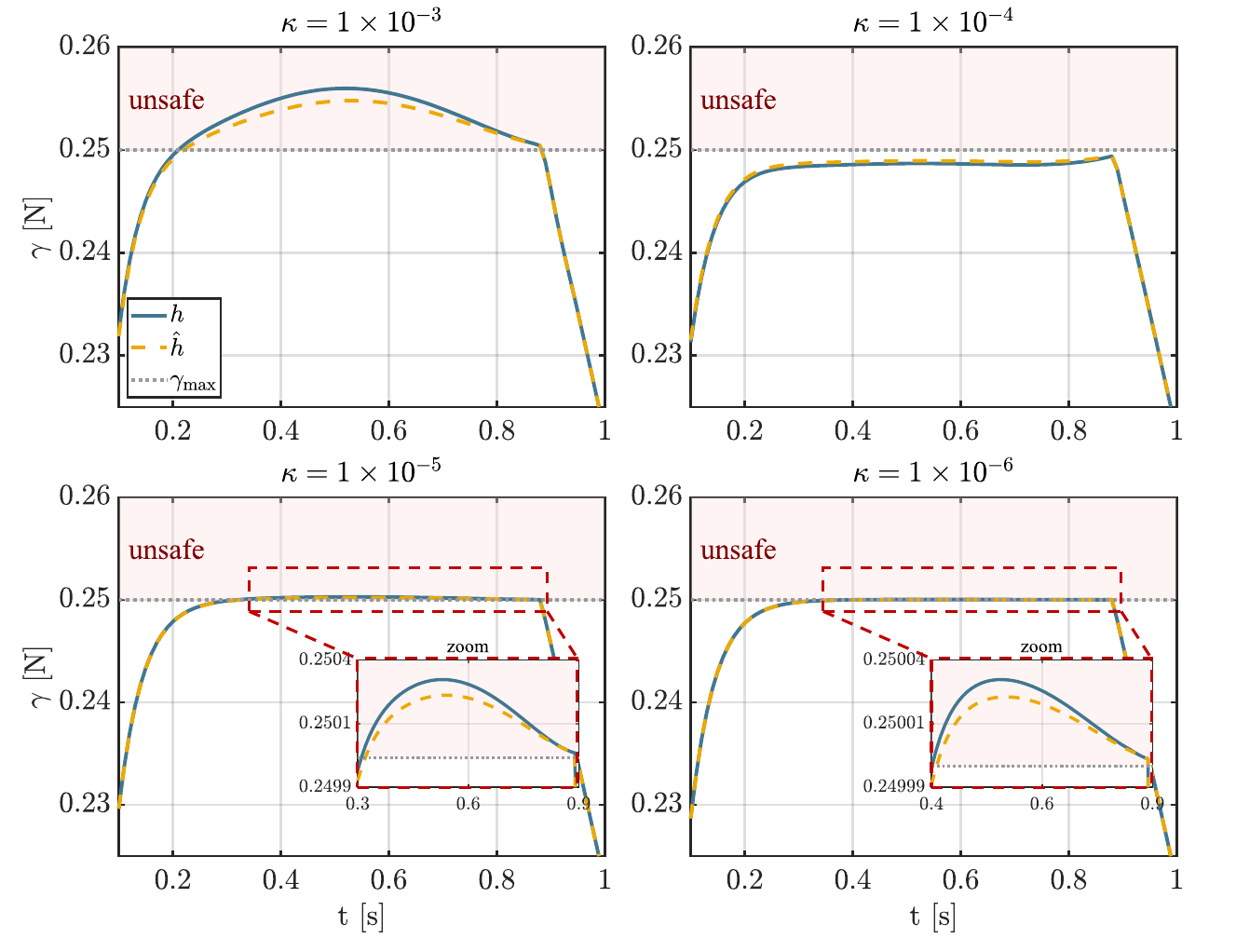}
    \caption{True and predicted contact forces near the safety boundary under the CBF-QP for different $\kappa$ values. The dotted gray line and shading indicate $\gamma_{\max}$ and the unsafe region $\gamma>\gamma_{\max}$, respectively.}
    \label{fig:kappa_cbf_test}
\end{figure}
Consequently, $A_\phi(\kappa)<1$ implies that decreasing $\kappa$ locally increases the true safety margin, whereas $A_\phi(\kappa)>1$ implies that decreasing $\kappa$ locally decreases the true safety margin. A first-order stationary point can occur when $A_\phi(\kappa)=1$.

This local identity shows that reducing $\kappa$ does not necessarily increase the true safety margin. A smaller $\kappa$ would improve the margin only if the closed-loop gap remained fixed. In closed loop, however, this gap also changes with $\kappa$ through the state, state increment, safety-filtered input, and smoothed contact update. This effect is captured by $A_\phi(\kappa)$. Therefore, the margin sensitivity has no fixed sign and depends on the sign of $A_\phi(\kappa)-1$.

The results in \figref{fig:kappa_sweep_test} are consistent with this interpretation. Since $h_{k+1}=\hat{h}_{k+1}+\Delta h_k$, true safety requires the predicted margin to compensate for any negative approximation error. 
As $\kappa$ decreases, $\Delta h$ approaches zero from below in (b), while the predicted margin can still be small near the safety boundary in (c). Thus, smaller $\kappa$ is not necessarily safer; for this system
and trajectory, an intermediate $\kappa$ provides the best observed balance between approximation accuracy and margin preservation.
\subsection{Boundary-Focused Screening of $\kappa$}\label{sec:kappa_screening}
The previous section shows that the true safety margin can vary non-monotonically with $\kappa$. Thus, $\kappa$ should not be regarded merely as a numerical smoothing parameter, since it also influences the relaxed contact response and the resulting closed-loop safety behavior. This motivates a boundary-focused compatibility screening procedure for $\kappa$, summarized in Algorithm~\ref{alg:kappa_screening}. The screening criterion compares the predicted margin with the observed one-step under-prediction near the safety boundary.
\begin{figure*}[t]
    \centering
    \includegraphics[width=0.95\linewidth]{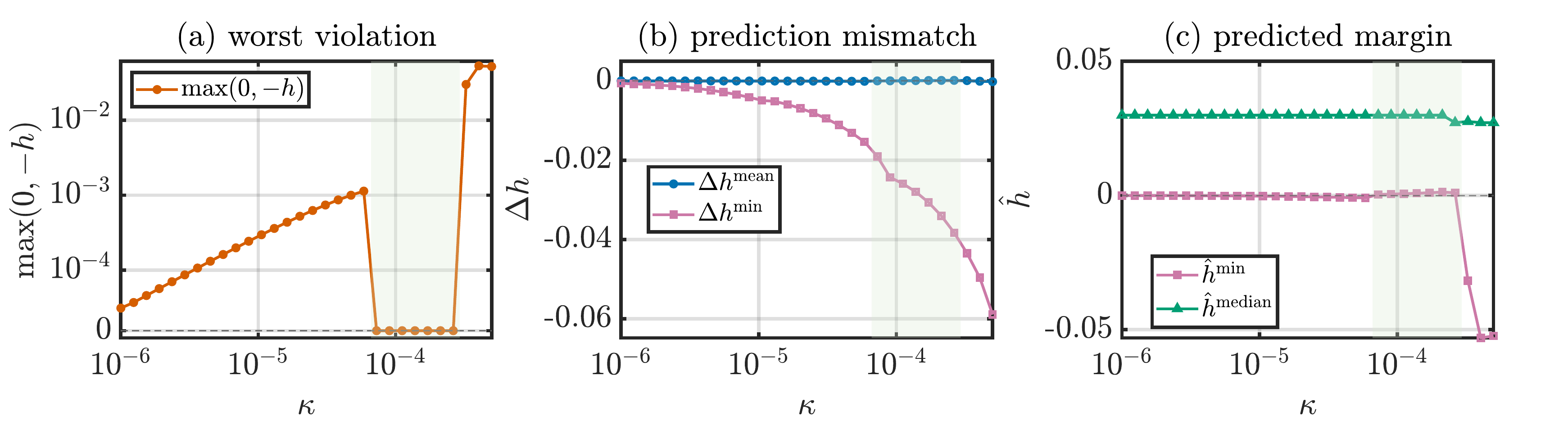}
    \caption{Safety-related metrics across the central-path parameter $\kappa$. (a) Worst-case violation magnitude, $\max(0,-h)$. (b) Approximation error, $\Delta h := h-\hat{h}$, showing its mean and minimum values. (c) Predicted safety margin with minimum and median values. The shaded region indicates the non-violation interval of $\kappa$ observed in simulation.}
    \vspace{-0.25cm}
    \label{fig:kappa_sweep_test}
\end{figure*}
\begin{algorithm}[h]
\caption{Boundary-Focused Screening of $\kappa$}
\label{alg:kappa_screening}
\begin{algorithmic}[1]
\Require Candidates $\mathcal K$, rollouts $\mathcal P$, $H,p,\beta$
\Ensure $\kappa^\star$
\For{$\kappa\in\mathcal K$}
    \State Run CBF rollouts on $\mathcal P$ and collect
    $\hat h_{k+1}, h_{k+1}$
    \State $\Delta h_k \gets h_{k+1}-\hat h_{k+1}$
    \State $\rho_\kappa \gets Q_{1-\beta}((-\Delta h_k)^+)$
    \State $S_p(\kappa) \gets Q_p(\hat h_{k+1})-\rho_\kappa$
    \State Compute $r_{\rm fail}(\kappa)$
\EndFor
\State $\mathcal K_{\rm comp}\gets
\{\kappa\in\mathcal K:r_{\rm fail}(\kappa)=0,\ S_p(\kappa)\ge 0\}$
\If{$\mathcal K_{\rm comp}\neq\emptyset$}
    \State $\mathcal K_{\rm score}\gets
    \arg\max_{\kappa\in\mathcal K_{\rm comp}} S_p(\kappa)$
    \State $\kappa^\star \gets \min \mathcal K_{\rm score}$
\Else
    \State $\mathcal K_{\rm fail}\gets
    \arg\min_{\kappa\in\mathcal K} r_{\rm fail}(\kappa)$
    \State $\mathcal K_{\rm score}\gets
    \arg\max_{\kappa\in\mathcal K_{\rm fail}} S_p(\kappa)$
    \State $\kappa^\star \gets \min \mathcal K_{\rm score}$
\EndIf
\State \Return $\kappa^\star$
\end{algorithmic}
\end{algorithm}
As illustrated in Fig.~\ref{fig:workflow}, the fixed rollout set $\mathcal{P}$ is constructed by perturbing task-specific states and inputs within the operating set $\mathcal{D}$ and simulating the resulting closed-loop trajectories over a short horizon $H$. For each candidate $\kappa$, the same scenarios are evaluated using the standard CBF-QP to collect the predicted margin $\hat{h}_{k+1}$ and true margin $h_{k+1}$ over all valid steps. Accordingly, the screening and margin calibration are local to the portion of $\mathcal{D}$ sampled by $\mathcal{P}$. Since safety depends on $h_{k+1}=\hat h_{k+1}+\Delta h_k$, negative values of $\Delta h_k$ indicate that the predicted margin is optimistic. We therefore define
\begin{align}
\rho_\kappa
&=
Q_{1-\beta}\!\left((-\Delta h_k(\kappa))^+\right), \\
S_p(\kappa)
&=
Q_p\!\left(\hat h_{k+1}(\kappa)\right)-\rho_\kappa,
\end{align}
where $Q_p(\cdot)$ denotes the empirical $p$-quantile of the collected rollout samples, and $(a)^+ := \max(a,0)$ denotes the positive-part operator. Here, $\rho_\kappa$ is the selected upper-tail under-prediction level, and $S_p(\kappa)$, referred to as the compatibility score, measures whether the lower-tail predicted margin is large enough to compensate for the observed one-step under-prediction. 

The parameters $p$ and $\beta$ determine the lower-tail margin quantile and the upper-tail under-prediction quantile, respectively. Here, $r_{\mathrm{fail}}(\kappa)$ denotes the fraction of rollouts with an implicit-solver failure, independently of $S_p(\kappa)$.
A candidate is called compatible if it has no rollout failures and its compatibility score is nonnegative. The compatible set is
\begin{equation}
\mathcal K_{\rm comp} = \left\{ \kappa\in\mathcal K: r_{\rm fail}(\kappa)=0,\;
S_p(\kappa)\ge 0 \right\}.
\end{equation}
\subsection{Robust Barrier Tightening Under Approximation Error}\label{sec:robust_cbf}

\begin{figure*}[h]
    \centering
     \includegraphics[width=0.95\linewidth]{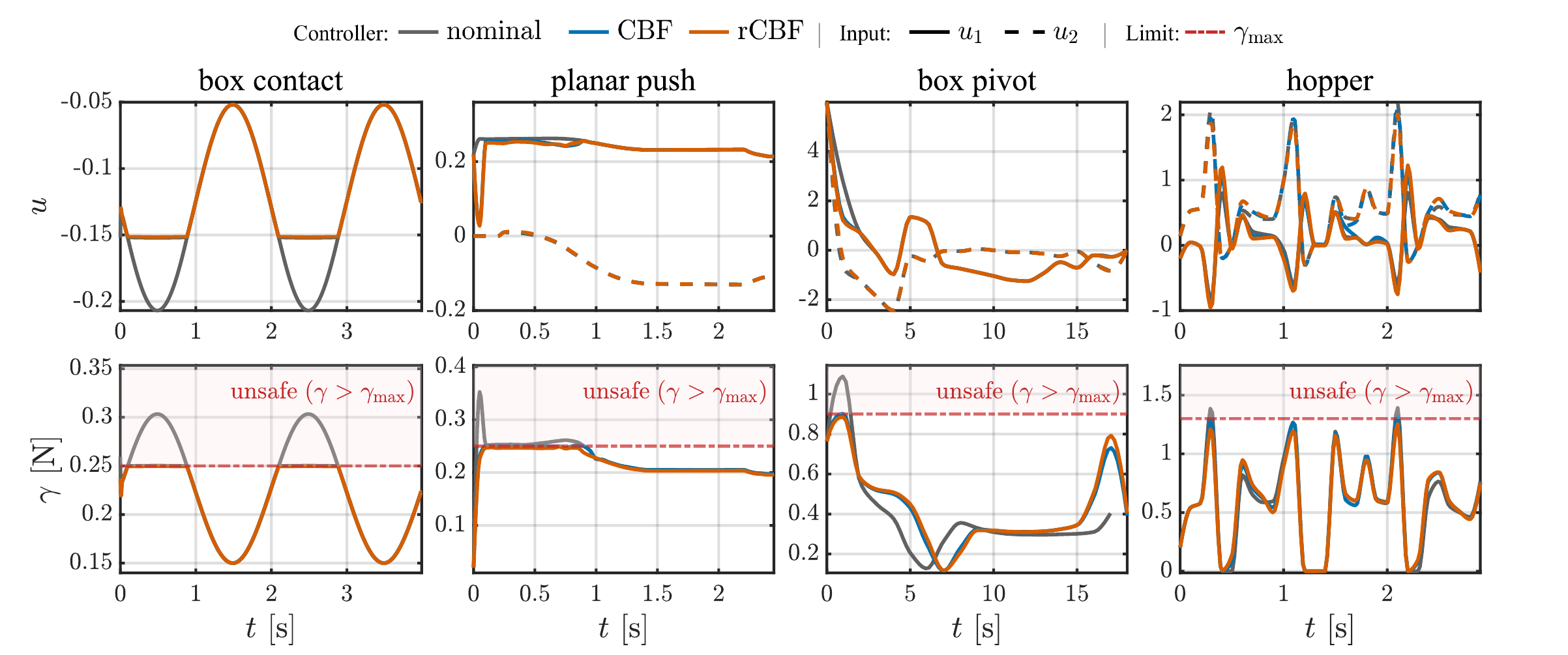}
    \caption{Control inputs $u$ (top) and safety metric $\gamma$ (bottom) for four systems. Gray, blue, and orange denote nominal, CBF, and rCBF, respectively; solid and dashed curves in the top row denote $u_1$ and $u_2$ (box contact uses scalar $u$). Red dash-dotted lines and shading indicate $\gamma_{\max}$ and unsafe regions ($\gamma > \gamma_{\max}$). Overlapping curves represent nearly identical commands. For box pivoting, the plotted contact force is the normal force at the active pivot contact.}
    \vspace{-0.25cm}
    \label{fig:simulations_cbf}
\end{figure*}

\begin{figure}[t]
    \centering
    \includegraphics[width=1.0\linewidth]{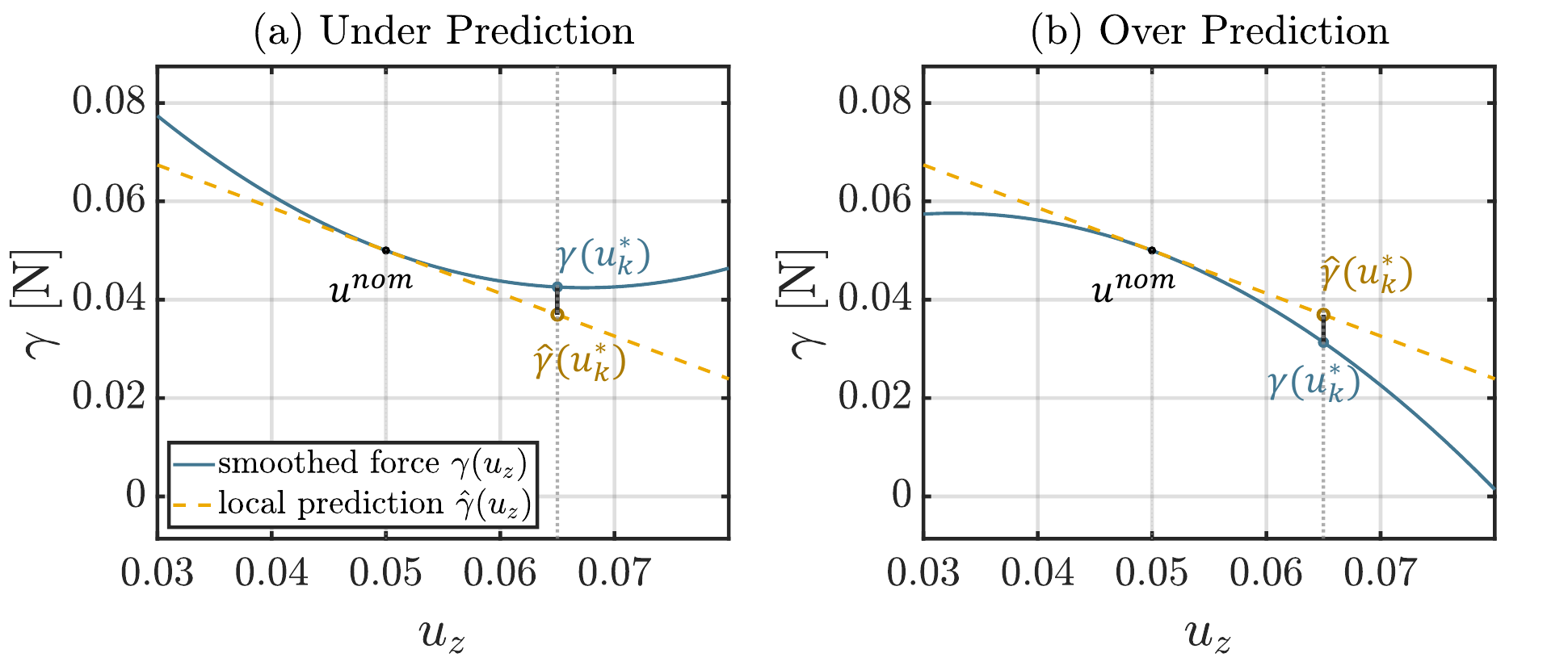}
    \caption{One-step force prediction error, where $\hat{\gamma}:=\gamma_{\max}-\hat{h}_{k+1}$ denotes the local force prediction around $u^{\mathrm{nom}}$. The prediction can under- or over-estimate the true force at \(u_k^\star\). Under-prediction is safety-critical because it may make an unsafe input appear feasible.} \label{fig:under_over}
\end{figure}

The screening procedure identifies a $\kappa$ whose predicted margin is compatible with the observed one-step under-prediction. However, it does not remove the residual approximation error introduced by the local Taylor approximation of the implicit contact force. Consequently, satisfying the nominal discrete-time CBF constraint based on the predicted safety margin does not guarantee that the true safety margin is nonnegative.

As illustrated in \figref{fig:under_over}, under-prediction can make an unsafe input appear feasible. Let $\delta_{\mathrm{lin}}\geq0$ denote the approximation-error tightening. We require
\begin{equation}
-\delta_{\mathrm{lin}} \le \Delta h_k
\label{eq:delta_h_bound}
\end{equation}
over an operating region near the safety boundary.

\begin{proposition}[Tightened CBF Constraint]
Suppose that the approximation error satisfies \eqref{eq:delta_h_bound}
for some $\delta_{\mathrm{lin}}\ge 0$. If the control input satisfies the
tightened predicted CBF condition
\begin{equation}
\hat{h}_{k+1}(u_k;\kappa)
\ge
(1-\alpha)h_k+\delta_{\mathrm{lin}},
\label{eq:robust_cbf_constraint_2}
\end{equation}
then the true next-step safety margin satisfies
\begin{equation}
h_{k+1} \ge (1-\alpha)h_k.
\end{equation}
In particular, if $h_k\ge0$ and $\alpha\in(0,1]$, then
$h_{k+1}\ge0$.
\end{proposition}

\begin{proof}
Using \eqref{eq:delta_h_bound}, \eqref{eq:robust_cbf_constraint_2}, and \eqref{eq:safety_h}, we obtain
\begin{equation}
h_{k+1} = \hat{h}_{k+1} + \Delta h_k
\ge
(1-\alpha)h_k+\delta_{\mathrm{lin}}+\Delta h_k
\ge
(1-\alpha)h_k.
\label{eq:true_barrier_guarantee}
\end{equation}
If $h_k\ge0$ and $\alpha\in(0,1]$, then
$(1-\alpha)h_k\ge0$, and hence $h_{k+1}\ge0$.
\end{proof}

Hence, when the approximation error bound in \eqref{eq:delta_h_bound} holds, the tightened constraint guarantees that the true safety margin
of the smoothed implicit model satisfies the nominal one-step barrier condition.

In practice, an exact closed-form bound for $\delta_{\mathrm{lin}}$ is generally unavailable. We therefore choose $\delta_{\mathrm{lin}}$ conservatively from boundary-focused offline rollouts using statistics of the observed under-prediction $(-\Delta h_k)^+$, such as an upper quantile or the maximum. 

As illustrated in \figref{fig:hero_figure}, the same tightening idea can be conditionally extended to the unsmoothed model when the smoothing mismatch is also bounded. Let $h_{k+1}^0$ denote the safety margin of the unsmoothed model, corresponding to $\kappa=0$. Then
\begin{equation}
    h_{k+1}^0= \hat{h}_{k+1}+\underbrace{(h_{k+1}-\hat{h}_{k+1})}_{\text{linearized error}}+\underbrace{(h_{k+1}^0-h_{k+1})}_{\text{smoothed mismatch}}.
\end{equation}
Using the same approximation-error margin $\delta_{\mathrm{lin}}$ and an additional smoothing-mismatch bound $\delta_{\mathrm{smooth}}$, if
\begin{align}
    -\delta_{\mathrm{lin}} &\le h_{k+1}-\hat{h}_{k+1},\\
    -\delta_{\mathrm{smooth}} &\le h_{k+1}^0-h_{k+1},
\end{align}
then enforcing
\begin{equation}
    \hat{h}_{k+1} \ge (1-\alpha)h_k^0 + \delta_{\mathrm{lin}} + \delta_{\mathrm{smooth}}
\end{equation}
implies
\begin{equation}
    h_{k+1}^0 \ge (1-\alpha)h_k^0.
    \label{eq:unsmoothed_contact}
\end{equation}
Accordingly, the QP directly constrains the local affine prediction $\hat{h}_{k+1}$. A valid linearized error bound transfers this condition to the smoothed implicit model margin $h_{k+1}$, while an additional smoothing mismatch bound conditionally extends it to the original unsmoothed model margin $h_{k+1}^0$.

The smoothing margin $\delta_{\mathrm{smooth}}$ can be interpreted through the impulse-momentum relation. Recall that $\gamma_{k+1}$ is the average contact force within time step $k$ of duration $T$. Due to the conservation of momentum, the difference between the unsmoothed force $\gamma^0_{k+1}$ and the smoothed force $\gamma^\kappa_{k+1}$ can be bounded as
\begin{equation}
    \gamma^0_{k+1}-\gamma^\kappa_{k+1}
    \le
    \frac{m_{\mathrm{eff}}}{T}
     \left|v^0_{n,k+1}-v^\kappa_{n,k+1}\right|,
     \label{eq:impulse_mismatch}
\end{equation}
where $m_{\mathrm{eff}}$ is the effective mass along the contact normal, and $v^0_{n,k+1}$ and $v^\kappa_{n,k+1}$ denote the post-contact normal velocities under the unsmoothed and smoothed contact dynamics, respectively. Assume that, over the considered state--input region $\mathcal{D}$, $m_{\mathrm{eff}}\leq\bar m_{\mathrm{eff}}$ and that the post-contact normal velocities of both models remain within a prescribed admissible range:
\begin{equation}
    v^0_{n,k+1},\,v^\kappa_{n,k+1}
    \in [\underline v_n,\overline v_n].
\end{equation}
Then,
\begin{equation}
    |v^0_{n,k+1}-v^\kappa_{n,k+1}|
    \leq \overline v_n-\underline v_n,
\end{equation}
and hence
\begin{equation}
    \gamma^0_{k+1}-\gamma^\kappa_{k+1}
    \leq
    \frac{\bar m_{\mathrm{eff}}}{T}
    (\overline v_n-\underline v_n)
    =:\delta_{\mathrm{smooth}}.
\end{equation}
Thus, the result in \eqref{eq:unsmoothed_contact} holds conditionally while these bounds remain valid over $\mathcal{D}$. In practice, $\bar m_{\mathrm{eff}}$ can be computed from the inertial model over $\mathcal{D}$, while $[\underline v_n,\overline v_n]$ can be selected conservatively from known state and input limits or offline rollouts; this calibration is not performed in the present study.
\section{Experiments and Results}\label{sec:Experiments}
We evaluate the proposed safety filter on the four contact-rich systems shown in \figref{fig:hero}: a 1D box contact, a planar push, a box pivot, and a hopper. The box-contact input is the vertical force $u_z$; the planar-push and box-pivot inputs are $u=[u_{p,x},u_{p,y}]$ and $u=[u_{p,x},u_{p,z}]$, respectively; and the hopper input is $u=[M,\dot r]$, where $M$ is the body-pitch torque and $\dot r$ is the leg-length rate command. Task-specific $\kappa$ screening follows Sec.~\ref{sec:kappa_screening}. Reference trajectories are generated via trajectory optimization following~\cite{howell2022trajectory} and tracked by a nominal PD controller. The local approximation and fixed robust margin are intended for the calibrated operating region; deviations beyond this region may require online replanning or additional safeguards.

 \begin{figure}[t]
    \centering
     \includegraphics[width=0.9\linewidth]{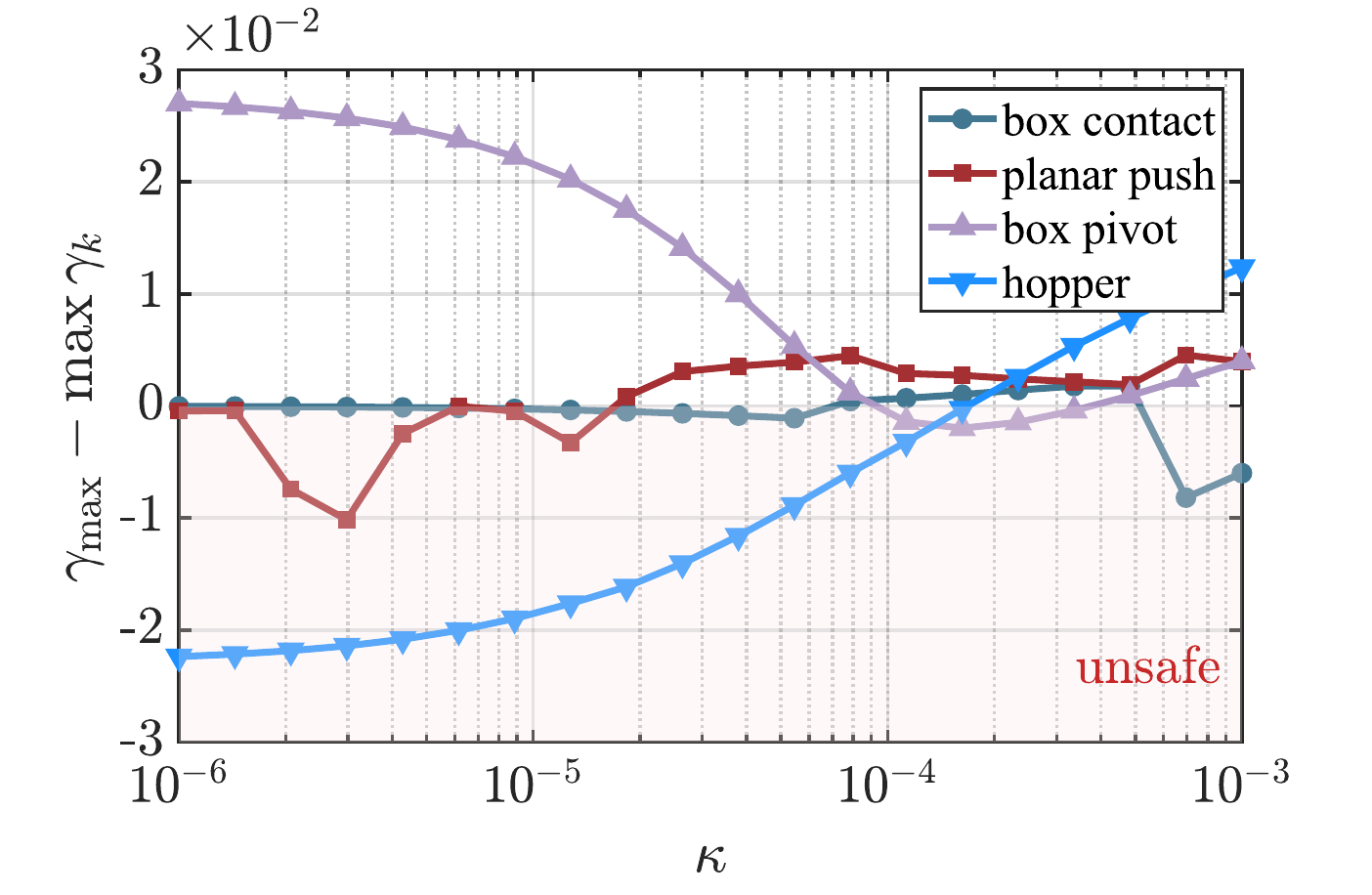}
     \caption{\small{$\kappa$-sweep results across four systems, showing the safety margin with unsafe regions shaded.}}
    \label{fig:kappa_sweep}
\end{figure} 

\subsection{$\kappa$-Dependent Safety Behavior}\label{sec:kappa_dependent}
Each system was evaluated at 20 logarithmically spaced $\kappa$ values in $[10^{-6},10^{-3}]$, with $\alpha=0.95$ fixed across systems, and checked for force-limit violations.
\figref{fig:kappa_sweep} shows that the relationship between $\kappa$ and safety margin varies substantially across systems and is often non-monotonic. In the box contact system, an intermediate band of $\kappa$ values leads to a safety violation, indicating that reducing $\kappa$ does not necessarily improve safety. The box pivot results show a similar pattern, with safe behavior at both smaller and larger $\kappa$ values but a safety violation near $\kappa \approx 10^{-4}$. The planar push also exhibits a non-monotonic trend, with substantial variation in safety margin across the tested range. In contrast, the hopper system shows a monotonic increase in safety margin as $\kappa$ increases. These results suggest that non-violating $\kappa$ values often form a range rather than a single optimal point, making $\kappa$ screening nontrivial in practice.

\subsection{Necessity of Robust CBF under Approximation Error}
We compare standard CBF and robust CBF control on the four systems using $\kappa$ from Algorithm~\ref{alg:kappa_screening} and $\alpha=0.95$. The selected $\kappa^\star$ and $\delta_{\mathrm{lin}}$ are fixed
before evaluation and are not tuned on the reported rollouts. 
In Table~I, the $(\kappa,\delta_{\mathrm{lin}})$ pairs are $(5.0\!\times\!10^{-5},5.0\!\times\!10^{-4})$, $(1.0\!\times\!10^{-4},2.4\!\times\!10^{-5})$, $(5.0\!\times\!10^{-4},1.9\!\times\!10^{-2})$, and $(3.0\!\times\!10^{-4},9.4\!\times\!10^{-2})$. Screening uses 40, 12, 8, and 40 rollout evaluations per candidate $\kappa$ over $H=120$, 12, 10, and 24 steps, respectively, with $p=\beta=0.1$.
These settings are specific to the reference trajectories and calibrated operating region.
The experiments instantiate only $\delta_{\mathrm{lin}}$ and concern the smoothed implicit dynamics; the conditional extension to unsmoothed dynamics requiring $\delta_{\mathrm{smooth}}$ is not evaluated.

\figref{fig:simulations_cbf} shows that the standard CBF reduces contact-force violations but may leave small residual violations due to approximation error even when the linearized CBF-QP constraint is satisfied. Table~\ref{tab:robust_quantitative} summarizes the quantitative comparison across all systems in terms of the violation rate (\textit{Viol.}), the maximum force overshoot above the limit (\textit{Max over.}), the peak contact force (\textit{Peak.}), and task success (\textit{Succ.}) for each controller. In these deterministic rollouts, the proposed robust CBF eliminates the observed contact force violations across all tested systems. It also improves task feasibility in the box-pivot and hopper examples, where the standard CBF fails to complete the task.

\begin{table}[t]
\centering
\caption{\small{Quantitative robust CBF comparison across four systems. Metrics are computed from deterministic rollouts for each controller. The force limits are $\gamma_{\max}=0.25$ (box contact), $0.245$ (planar push), $0.9$ (box pivot), and $1.3$ (hopper). Succ.: completion of the intended tip-over motion for box pivoting or forward traversal without falling for the hopper; N/A: no binary completion criterion is defined for the box-contact and planar-push examples.}}
\label{tab:robust_quantitative}
\footnotesize
\begin{tabular}{llcccc}
\toprule
System & Controller & Viol. & Max over. & Peak. & Succ. \\
\midrule
\multirow{3}{*}{Box contact} & Nominal & 0.398 & 0.0533 & 0.3033 & N/A \\
 & CBF & 0.345 & 0.0002 & 0.2502 & N/A \\
 & rCBF & 0.000 & 0.0000 & 0.2497 & N/A \\
\midrule
\multirow{3}{*}{Planar push} & Nominal & 0.340 & 0.0799 & 0.3249 & N/A \\
 & CBF & 0.260 & 0.0000 & 0.2450 & N/A \\
 & rCBF & 0.000 & 0.0000 & 0.2450 & N/A \\
\midrule
\multirow{3}{*}{Box pivot} & Nominal & 0.056 & 0.1890 & 1.0890 & Y \\
 & CBF & 0.053 & 0.0026 & 0.9026 & N \\
 & rCBF & 0.000 & 0.0000 & 0.8851 & Y \\
\midrule
\multirow{3}{*}{Hopper} & Nominal & 0.067 & 0.1017 & 1.4017 & N \\
 & CBF & 0.033 & 0.0117 & 1.3117 & N \\
 & rCBF & 0.000 & 0.0000 & 1.2225 & Y \\
\midrule
\bottomrule
\end{tabular}
\end{table}

\begin{table}[t]
\centering
\caption{\small{Sensitivity to fixed robust barrier tightening $\delta_{\mathrm{lin}}$ on the hopper and planar push systems. The columns $\rho_{85}$ and $\rho_{90}$ denote the 85th and 90th percentiles of the one-step under-prediction samples $(-\Delta h_k)^+$ used to choose $\delta_{\mathrm{lin}}$, and $\max$ denotes the maximum observed under-prediction.}}
\label{tab:delta_sensitivity_combined}
\footnotesize
\begin{tabular}{llcccc}
\toprule
System & Metric & $0$ & $\rho_{85}$ & $\rho_{90}$ & $\max$ \\
\midrule
\multirow{3}{*}{Hopper}
& Peak  & 1.3117 & 1.2713 & 1.2560 & 1.2186 \\
& Viol. & 0.033  & 0.000  & 0.000  & 0.000 \\
& Dev.  & 0.0688 & 0.0827 & 0.0866 & 0.1015 \\
\midrule
\multirow{3}{*}{Planar push}
& Peak  & 0.2450 & 0.2450 & 0.2450 & 0.2450 \\
& Viol. & 0.260  & 0.020  & 0.000  & 0.000 \\
& Dev.  & 0.0055 & 0.0055 & 0.0055 & 0.0055 \\
\midrule
\bottomrule
\end{tabular}
\end{table}

Table~\ref{tab:delta_sensitivity_combined} shows the trade-off induced by the fixed robust margin $\delta_{\mathrm{lin}}$ on the hopper and planar push systems. Here, \textit{Dev.} denotes the mean control deviation from the nominal input over the rollout, measured as $\|u-u^{\mathrm{nom}}\|$. In both systems, increasing $\delta_{\mathrm{lin}}$ improves safety but makes the control action more conservative. The required margin is system-dependent. On the hopper, $\rho_{85}$ is sufficient to eliminate force violations, whereas on the planar push, residual violations remain at $\rho_{85}$ and the tested rollout first becomes violation-free at $\rho_{90}$.




\section{Conclusion}\label{sec:discussion}
In this work, we presented a safety-critical control method for smoothed implicit contact dynamics by locally approximating the implicit contact force with a first-order Taylor expansion. Our analysis revealed that $\kappa$, while improving numerical smoothness, has a nontrivial effect on safety. Although smaller $\kappa$ reduces one-step approximation error, the resulting closed-loop safety behavior is non-monotonic because $\kappa$ changes both the relaxed contact response and the safety-filtered closed-loop trajectory.

These observations motivate two practical components: boundary-focused $\kappa$-screening to identify task-compatible values within the considered operating region, and a robust CBF with a tightening margin to reduce approximation-induced violations.  For hardware deployment, the rollout set $\mathcal{P}$
may be generated in an identified simulator using conservative
state, input, and parameter ranges, without requiring near-limit
hardware rollouts. The selected $\kappa$ and rollout-based
margin should therefore be interpreted as applicable only within
the calibrated operating region. Residual sim-to-real mismatch
and operation outside this region require additional validation,
recalibration, or safeguards. These results show that contact smoothing should be treated not only as a numerical device but also as a safety-relevant design parameter in contact-rich control. We expect this perspective to support the more reliable use of implicit contact dynamics in safety-critical robotic applications.

\section*{Acknowledgement}
This research was supported by Pioneer Center for Accelerating P2X Materials Discovery (CAPeX), DNRF grant number P3 and Fabrikant Vilhelm Pedersen og Hustrus Legat.


\bibliographystyle{IEEEtran}
\bibliography{bibliography}
\end{document}